\documentclass[lettersize,journal]{IEEEtran}
\usepackage{amsmath,amsfonts}
\usepackage{algorithmic}
\usepackage{algorithm}
\usepackage{array}
\usepackage{amsthm}
\newtheorem{proposition}{Proposition}
\usepackage[caption=false,font=normalsize,labelfont=sf,textfont=sf]{subfig}
\usepackage{textcomp}
\usepackage{stfloats}
\usepackage{url}
\usepackage{verbatim}
\usepackage{graphicx}
\usepackage{cite}
\usepackage{url}
\usepackage[hidelinks]{hyperref}
\usepackage[utf8]{inputenc}
\usepackage{amsmath}
\usepackage{amsthm}
\usepackage{booktabs}
\usepackage{algorithm}
\usepackage[switch]{lineno}
\usepackage{wrapfig}
\usepackage{subcaption}
\usepackage{multirow}
\usepackage{makecell}
\usepackage{xcolor}
\newcommand{\mathbold}[1]{\ensuremath{\boldsymbol{\mathbf{#1}}}}

\newcommand{\g}{\,|\,}

\newcommand{\nestedmathbold}[1]{{\mathbold{#1}}}

\newcommand{\mbc}{\nestedmathbold{c}}

\newcommand{\mbe}{\nestedmathbold{e}}

\newcommand{\mbh}{\nestedmathbold{h}}

\newcommand{\mbI}{\nestedmathbold{I}}

\newcommand{\mbtheta}{\nestedmathbold{\theta}}

\newcommand{\mbSigma}{\nestedmathbold{\Sigma}}

\newcommand{\coloredstate}[1]{\STATE \textcolor{gray}{#1}}

\begin{document}

\title{Marked Temporal Bayesian Flow Point Processes}

\author{Hui Chen, Xuhui Fan, Hengyu Liu and Longbing Cao,~\IEEEmembership{Senior Member,~IEEE}
\thanks{The work is partially sponsored by Australian Research Council Discovery, LIEF and Future Fellowship grants (DP190101079, DP240102050, FT190100734 and FT190100734).}
\thanks{Hui Chen, Xuhui Fan and Longbing Cao are with School of Computing, Macquarie University, Australia. (email: hui.chen2@students.mq.edu.au, xuhui.fan@mq.edu.au, longbing.cao@mq.edu.au).
Hengyu Liu is with Aalborg University, Denmark. (email: heli@cs.aau.dk).}}



\maketitle

\begin{abstract}
Marked event data captures events by recording their continuous-valued occurrence timestamps along with their corresponding discrete-valued types.  They have appeared in various real-world scenarios such as social media, financial transactions, and healthcare records, and have been effectively modeled through Marked Temporal Point Process (MTPP) models. Recently, developing generative models for these MTPP models have seen rapid development due to their powerful generative capability and less restrictive functional forms. However, existing generative MTPP models are usually challenged in jointly modeling events' timestamps and types since: (1) mainstream methods design the generative mechanisms for timestamps only and do not include event types; (2) the complex interdependence between the timestamps and event types are overlooked. In this paper, we propose a novel generative MTPP model called BMTPP. Unlike existing generative MTPP models, BMTPP flexibly models marked temporal joint distributions using a parameter-based approach. Additionally, by adding joint noise to the marked temporal data space, BMTPP effectively captures and explicitly reveals the interdependence between timestamps and event types. Extensive experiments validate the superiority of our approach over other state-of-the-art models and its ability to effectively capture marked-temporal interdependence. 
\end{abstract}


\section{Introduction}
\IEEEPARstart{M}{arked} event data is widely seen in the real world as a sequence of events, where each event is recorded with a continuous-valued occurrence timestamp and a categorical event type (a.k.a. mark). Its detailed applications include social media \cite{zhao2015seismic,farajtabar2017coevolve}, where different event types may trigger diverse event patterns; financial transactions \cite{bacry2014hawkes,hawkes2018hawkes}, where a buy or sell action would result in different transaction times; and healthcare records \cite{wang2016isotonic,meng2023dynamic}, where the disease types decide the visit times. That is, timestamps and event types exhibit non-trivial interdependence in these scenarios, as they can influence each other in both directions, specified by the detailed situation.

Marked Temporal Point Process (MTPP) is a stochastic process that can effectively model marked event sequences. Mainstream MTPP methods can be broadly divided into two categories: classical MTPP models, including Poisson processes \cite{kingman1992poisson,wu2020modeling}, Hawkes processes \cite{hawkes1971spectra} and Self-correcting processes \cite{isham1979self}, which use an intensity function to characterize the instantaneous occurrence of events given their history. However, these methods often rely on strong parametric or modeling assumptions, limiting their ability to effectively capture complex patterns in event occurrences. On the other hand, neural MTPP models have emerged as a rapidly developing branch in recent years \cite{shchur2021neural,xiao2019learning,wang2020spatio}. These models employ neural networks, such as RNNs, LSTMs, or transformers, to encode historical events \cite{xue2023easytpp}.  In some cases, they draw on intensity functions inspired by processes like the Hawkes process to model event occurrences \cite{zuo2020transformer,zhang2020self}. Compared to classical MTPP models, neural MTPP models leverage the expressive power of neural networks to better capture the complex dependencies among events. However, applying these two methods to parametric MTPP models requires solving integrals to compute the likelihood function, which usually requires extremely high computational cost \cite{shchur2021neural,chen2020neural}. 

To tackle this issue, two common techniques, namely strong model assumptions and numerical approximation techniques, are typically employed. First, certain model assumptions are introduced, such as treating the timestamp $x$ and event type $m$ as independent or conditionally dependent (e.g., $x$ depends on $m$, or 
$m$ depends on $x$) \cite{du2016recurrent,xiao2017modeling,zuo2020transformer,bilovs2019uncertainty}. This approach simplifies the integral form, thereby reducing computational complexity.  However, in reality, $x$ and $m$ may not be independent or may have more complex dependence, which can lead to model misspecification and consequently restrict the model's expressive capacity \cite{enguehard2020neural}. Second, numerical approximation techniques, such as Monte Carlo sampling or numerical integration, are used to simplify the computation of integrals when closed-form solutions are not available \cite{mei2017neural,nickel2020learning}. Despite this, limitations in sample size and sampling errors mean that these methods only approximate the true solution, which can result in information loss and affect model performance \cite{chen2020neural}.

To fill these gaps, generative MTPP models, which model the target distribution of timestamps using the generative models, have been proposed and have shown promising results \cite{lin2022exploring,yuan2023spatio}. Mainstream generative models, such as diffusion models \cite{sohl2015deep,ho2020denoising}, generative adversarial networks (GANs) \cite{goodfellow2014generative,kang2023cm}, and variational autoencoders (VAEs) \cite{kingma2013auto,pan2023vae}, are energy-based deep probabilistic models, where the optimization objective is an energy function corresponding to an unnormalized negative log-likelihood function. Typically, neural networks are employed to represent this energy function. Since neural networks can approximate complex functions without formal constraints, there is no need to impose model assumptions or use numerical approximation techniques to simplify computation. Consequently, using such generative models in MTPP tasks allows for flexible modeling, enhancing the model's expressiveness and avoiding information loss due to approximation operations. However, generative MTPP models still face two main challenges.

\textit{Challenge I}: In MTPP tasks, we aim to model the joint distribution $p(x,m)$ of two heterogeneous random variables: $x$, which is continuous, and $m$, which is discrete. However, mainstream generative models, such as diffusion models, are designed for continuous random variables due to their reliance on Gaussian noise \cite{song2020score}. As a result, these models are not directly applicable for modeling joint distributions that include discrete random variables $m$ \cite{lin2022exploring}.

\textit{Challenge II}: The target joint distribution $p(x,m)$ involves two random variables, $x$ and $m$, which exhibit a strong interdependence across different scenarios. For example, discussions about clothing types change with the seasons \cite{xue2023easytpp}. Thus, capturing the complex interdependence between timestamps $x$ and different event types $m$ is crucial for improving model performance. However, existing generative MTPP models are unable to directly model the joint distribution, often leading to the assumption that $x$ and $m$ are independent, and applying the generative model only to $x$ \cite{lin2022exploring}. This approach neglects the interdependence between the two random variables, ultimately compromising model performance.

To tackle the above challenges, we propose a novel generative MTPP model called the \underline{M}arked \underline{T}emporal \underline{B}ayesian Flow \underline{P}oint \underline{P}rocess (BMTPP), based on the recently developed generative model, the Bayesian Flow Network (BFN) \cite{graves2023bayesian}, to approximate the joint distribution of timestamps and event types. First, in contrast to existing generative MTPP models that model the continuous random variable $x$ only, BMTPP flexibly leverages BFN in a parameter-based manner to model both the timestamps and event types. Second, by adding the joint noise to the marked temporal data, BMTPP can effectively capture the coupling correlation between timestamps and event types in MTPP tasks, explicitly revealing the complex interactions between them. We summarize the contributions as follows:
\begin{itemize}
    \item Based on the Bayesian flow network, we propose a new generative MTPP model, i.e., BMTPP, which can naturally and flexibly model marked event sequences.
    \item BMTPP can directly model the joint distribution $p(x,m)$, which consists of continuous timestamps $x$ and discrete event types $m$. Moreover, by employing a joint noise strategy within the marked temporal data space, it can effectively capture and explicitly reveal the interdependence between timestamps and event types.
    \item We evaluate BMTPP on four real-world datasets, demonstrating our proposed approach outperforms other state-of-the-art models, as well as the effectiveness of our joint noise strategy in capturing marked-temporal interdependence.
\end{itemize}

\section{Preliminaries}
\subsection{Marked Temporal Point Processes}
A marked temporal point process (MTPP) \cite{rasmussen2018lecture} is a stochastic process whose realization is a set of marked events $\mathcal{S}=\{\mathbf{g}_i\}_{i=1}^{N}$, where $\mathbf{g}_i=\{x_i,m_i\}$. Here, the $x_i \in \mathbb{R}^+$ denotes the timestamp  with $x_i<x_{i+1}$, and the $m_i \in \{1,2, \dots, M\}$ is the corresponding event type for $i$-th event. The historical event sequence that occurred before $x$ is represented as $\mathcal{H}_{x}=\{(x_j,m_j)|j:x_j<x\}$.  We denote the conditional intensity function for MTPP as
\begin{equation}
\begin{aligned}
\lambda^*(x,m) & =\lambda(x, m \mid \mathcal{H}_x) \\
& =\lim _{\Delta x \rightarrow 0^{+}} \frac{\mathbb{E}\left[\mathcal{N}(x+\Delta x, m)-\mathcal{N}(x, m) \mid \mathcal{H}_x\right]}{\Delta x},
\end{aligned}
\end{equation}
which represents the expected instantaneous rate of event occurrences given the history \cite{lin2021empirical}.

 Given the conditional intensity function, the joint distribution $p^*(x,m)$ can then be defined as
\begin{equation}
p^*(x,m)=\lambda^*(x,m) \exp \left(-\int_{x_{i-1}}^x \lambda^*(s,m) d s\right),
\end{equation}
where the calculation of the integral term is challenging, thus requiring the use of model assumptions or numerical approximation techniques to resolve it \cite{shchur2021neural,chen2020neural}.

To encode the historical event sequence $\mathcal{H}_{x}$ into a fixed-dimensional historical embedding $\mathbf{h}_{i-1}$, we first lift the timestamp and event type information of $j$-th historical event in a high-dimensional space as
\begin{equation}
\mathbf{e}_j=[\mathbf{w}(x_j);\mathbf{E}^T\mathbf{m}_j],
\end{equation} 
where $\mathbf{w}$ transforms the one-dimensional $x_j$ to a high-dimensional vector; $\mathbf{E}$ is the embedding matrix for event types, and $\mathbf{m}_j$ is the one-hot encoding for event type $m_j$ \cite{lin2021empirical}. The historical encoder $\mathbf{W}$ can then be used for transforming the sequence $\{\mathbf{e}_1, \mathbf{e}_2, \dots \mathbf{e}_{i-1}\}$ into a fixed $D$-dimensional vector space by
\begin{equation}
\mathbf{h}_{i-1}=\mathbf{W}(\{\mathbf{e}_1, \mathbf{e}_2, \dots \mathbf{e}_{i-1}\}).
\end{equation}
Here, $\mathbf{W}$ could be the recurrent encoders, such as RNN and LSTM \cite{du2016recurrent,mei2017neural}, or the set aggregation encoders, such as attention-based models. \cite{zhang2020self,zuo2020transformer}.

\subsection{Bayesian Flow Networks}
The Bayesian Flow Network (BFN) \cite{graves2023bayesian} is a new generative model that modifies distribution parameters through Bayesian inference based on samples with varying levels of noise. The training process can be described as the $K$-step information transmission process between Alice (sender) and Bob (receiver), which can be summarized as the following procedure:

\textbf{Step 1: Information Transmission.} Alice directly interacts with the data and transmits part of its information to Bob. This part of the information is contained in the sample from the sender distribution obtained by directly adding noise to the data. For a $D$-dimensional data $\mathbf{g}$, the \textit{sender distribution} can be formulated as:
\begin{equation}
p_S(\mathbf{y}|\mathbf{g};\alpha)=\prod_{d=1}^Dp_S(y_d|g_d;\alpha),
\end{equation}
where the $\mathbf{y}=\{y_1,y_2,\dots,y_D\}$ is the noisy version of $\mathbf{g}$, and $\alpha$ is the predefined accuracy associated with noise levels.

\textbf{Step 2: Cognition Enhancement.} Bob then receives this information and enhances his cognitive abilities. For each information transmission step, Bob receives the sample from the sender distribution and updates its input distribution with a Bayesian update to enhance his cognitive abilities. The \textit{Bayesian update distribution} for updated parameters $\boldsymbol{\theta}_k$ is given by
\begin{equation}
p_U(\boldsymbol{\theta}_k|\boldsymbol{\theta}_{k-1},\mathbf{g};\alpha)=\mathbb{E}_{p_S(\mathbf{y}|\mathbf{g};\alpha)}\delta(\boldsymbol{\theta}_k-h(\boldsymbol{\theta},\mathbf{y},\alpha)),   
\end{equation}
where the $\delta(\cdot)$  and $h(\cdot)$ denote the Dirac delta function and Bayesian update function. In the actual execution of BFN, we generalize K steps to an infinite number of steps to directly obtain the \textit{Bayesian flow distribution} of updated parameters $\boldsymbol{\theta}$ for each time step $t=k/K$ as follows:
\begin{equation}
p_F(\boldsymbol{\theta}|\mathbf{g};t)=p_U(\boldsymbol{\theta}|\boldsymbol{\theta}_0,\mathbf{g};\beta(t)), 
\end{equation}
where $\beta(t)=\int_{t'=0}^t \alpha(t')dt'$ is the accuracy schedule.

\textbf{Step 3: Information Inference.} With his improved cognitive abilities, Bob infers the information Alice sent and further refines his inference ability. The information that Bob inferred is reflected in the \textit{receiver distribution}, which is formalized as
\begin{equation}
p_R(\mathbf{y}|\boldsymbol{\theta};t,\alpha)=\mathbb{E}_{p_O(\mathbf{g}|\boldsymbol{\theta};t)}p_S(\mathbf{y}|\mathbf{x};\alpha).    
\end{equation}
Here the $p_O(\mathbf{g}|\boldsymbol{\theta};t)=\prod_{d=1}^D p_O(g_d|\Psi_d(\boldsymbol{\theta},t))$ represent the output distribution and $\Psi(\cdot)$ is a neural network corresponding to Bob’s inference ability. The training objective of BFN is to refine Bob's inference ability by minimizing the discrepancy between the sender distribution and receiver distribution over the entire information transmission process: 
\begin{equation}
\label{obje}
\begin{aligned}
L^{\infty}(\mathbf{g})=&\lim_{\epsilon\to0}\frac{1}{\epsilon}\mathbb{E}_{t\sim U(\epsilon,1),p_F(\boldsymbol{\theta}|\mathbf{g},t-\epsilon)}\\
&D_{KL}(p_S(\mathbf{y}|\mathbf{g};\alpha(t,\epsilon))\|p_R(\mathbf{y}|\boldsymbol{\theta};t-\epsilon,\alpha(t,\epsilon))),
\end{aligned}
\end{equation}
where $\epsilon\stackrel{\text{def}}{=} \frac{1}{K}$ and $\alpha(t,\epsilon)\stackrel{\text{def}}{=}\beta(t)-\beta(t-\epsilon)$.

When the training process is finished, given the prior parameter $\boldsymbol{\theta}_0$ and accuracies $\alpha_1, \dots, \alpha_n$, we can draw the samples from $p_O(\mathbf{g}|\boldsymbol{\theta}, t)$ recursively.

\section{Marked Temporal Bayesian Flow Point Processes}
As in previous studies \cite{lin2022exploring,yuan2023spatio,shchur2021neural}, we consider modeling time intervals $\tau$ directly rather than timestamps $x$, we can then formulate the marked event as $\mathbf{g}=\{\tau, m\}$. In this section, we first give a detailed formulation for continuous random variable $\tau$ and discrete random variable $m$, respectively. The joint noise strategy and algorithms are then provided.   

\subsection{Formulation of Time Intervals $\tau$} 
For the continuous random variables $\tau$, the input distribution is defined as a Gaussian distribution $p_I(\tau|\boldsymbol{\theta}^{(\tau)})=\mathcal{N}(\tau|\mu,\rho^{-1})$, where $\boldsymbol{\theta}^{(\tau)}:=\{\mu,\rho\}$, and the parameters of the prior distribution are set as $\boldsymbol{\theta}_0^{(\tau)}=\{0,1\}$. The sender distribution is represented as 
\begin{equation}
\label{con_sen}
p_S(y^{(\tau)}|\tau;\alpha)=\mathcal{N}(\tau,\alpha^{-1}).
\end{equation}

With a noisy sample $y^{(\tau)}$, we have the Bayesian update function using the favorable property of the Gaussian distribution:
\begin{equation}
h(\{\mu_{k-1},\rho_{k-1}\},y^{(\tau)},\alpha)=\{\mu_k,\rho_k\},
\end{equation}
where the $\rho_k=\rho_{k-1}+\alpha$ and $\mu_k=\frac{\mu_{k-1}\rho_{k-1}+y^{(\tau)}\alpha}{\rho_k}$.

Given that $\mu_k$ represents the only random component of $\boldsymbol{\theta}_k^{(\tau)}$, the Bayesian update distribution is written by
\begin{equation}
p_U(\boldsymbol{\theta}_k^{(\tau)}|\boldsymbol{\theta}_{k-1}^{(\tau)},\tau;\alpha)=\mathcal{N}(\mu_k|\frac{\alpha\tau+\mu_{k-1}\rho_{k-1}}{\rho_k},\frac{\alpha}{\rho^2_{k-1}}).
\end{equation}

To generalize the discrete $K$ steps to the continuous infinite steps, we have the corresponding Bayesian flow distribution as
\begin{equation}
\label{flow}
p_F(\boldsymbol{\theta}^{(\tau)}|\tau;t)=p_U(\boldsymbol{\theta}^{(\tau)}|\boldsymbol{\theta}^{(\tau)}_0,\tau;\beta(t)).
\end{equation}

For the continuous random variable, the neural network $\Psi(\cdot)$ takes the $\boldsymbol{\theta}^{(\tau)}$ as the input, and the output is the estimated $\hat{\tau}=\Psi(\boldsymbol{\theta}^{(\tau)},t)$ \footnote{For the sake of readability, we have simplified the process of deriving estimated output $\hat{\tau}$. For further details, please refer to \cite{graves2023bayesian}.}. Thus the resulting output distribution can be given by
\begin{equation}
\label{con_out}
p_O(\tau|\boldsymbol{\theta}^{(\tau)};t)=\delta(\tau-\Psi(\boldsymbol{\theta}^{(\tau)},t)).
\end{equation}

Combining the \eqref{con_sen} and \eqref{con_out} together, the receiver distribution can be derived as 
\begin{equation}
\begin{aligned}
p_R(y^{(\tau)}|\boldsymbol{\theta}^{(\tau)};t,\alpha)&=\underset{\delta(\tau-\Psi(\boldsymbol{\theta}^{(\tau)},t))}{\mathbb{E}} \mathcal{N} (y^{(\tau)}|\tau,\alpha^{-1})\\
&=\mathcal{N}(y^{(\tau)}|\hat{\tau},\alpha^{-1}).
\end{aligned}
\end{equation}

\subsection{Formulation of Marks $m$}
For the discrete random variables $m$, we have $m\in [M]$, where $[M]$ represents the set of integers from $1$ to $M$. Therefore, the input distribution of $m$ is defined as a categorical distribution with parameter $\boldsymbol{\theta}^{(m)}=(\theta^{(m)}_1,\dots,\theta^{(m)}_M) \in \Delta^{M-1}$, where the $\theta_m^{(m)} \in [0,1]$. The input prior is set to be uniform with $\boldsymbol{\theta}_0^{(m)}=\frac{\mathbf{1}}{\mathbf{M}}$, which is the length $M$ vector whose entries are all $\frac{1}{M}$. The sender distribution can be obtained based on the central limit theorem as
\begin{equation}
\label{dis_sen}
p_S(y^{(m)}|m;\alpha)=\mathcal{N} (\alpha (M\mathbf{e}_{m}-\mathbf{1}),\alpha M \mathbf{I}),
\end{equation}
where the $\mathbf{1}$ denotes the vector of ones, $\mathbf{I}$ is the identity matrix and $\mathbf{e}_m\in \mathbb{R}^{M}$ is a vector obtained by mapping class index $m$ to a one-hot vector of length $M$ \cite{graves2023bayesian}.

Based on the proof of \cite{graves2023bayesian}, we can denote the Bayesian update function as
\begin{equation}
h(\boldsymbol{\theta}^{(m)}_{k-1},y^{(m)},\alpha):=\frac{e^{y^{(m)}}\boldsymbol{\theta}^{(m)}_{k-1}}{\sum_{m=1}^M e^{y^{(m)}_m}(\boldsymbol{\theta}^{(m)}_{k-1})_m}.
\end{equation}

The Bayesian update function can be provided by
\begin{equation}
\begin{aligned}
&p_U(\boldsymbol{\theta}^{(m)}|\boldsymbol{\theta}^{(m)}_{k-1},m;\alpha)\\
&=\underset{\mathcal{N} (\alpha (M\mathbf{e}_{m}-\mathbf{1}),\alpha M \mathbf{I})}{\mathbb{E}} \delta\bigg(\boldsymbol{\theta}^{(m)}-\frac{e^{y^{(m)}}\boldsymbol{\theta}^{(m)}_{k-1}}{\sum_{m=1}^M e^{y^{(m)}_m}(\boldsymbol{\theta}^{(m)}_{k-1})_m}\bigg).
\end{aligned}
\end{equation}

Similar to \eqref{flow}, to generalize the discrete $K$ steps to the continuous infinite steps, the Bayesian flow distribution is given by
\begin{equation}
\begin{aligned}
&p_F(\boldsymbol{\theta}^{(m)}|m;t)\\
&=\underset{\mathcal{N} (\beta(t) (M\mathbf{e}_{m}-\mathbf{1}),\beta(t) M \mathbf{I})}{\mathbb{E}} \delta(\boldsymbol{\theta}^{(m)}-\mathrm{softmax}(y^{(m)})).
\end{aligned}
\end{equation}

For the discrete random variable, the neural network $\Psi(\cdot)$ inputs the $\boldsymbol{\theta}^{(m)}$ and outputs the $\Psi(\boldsymbol{\theta}^{(m)},t) \in \mathbb{R}^M$, the corresponding output distribution can be derived as 
\begin{equation}
\label{dis_out}
p_O(m|\boldsymbol{\theta}^{(m)};t)=(\mathrm{softmax}(\Psi(\boldsymbol{\theta}^{(m)},t)))_m.
\end{equation}

Combining the \eqref{dis_sen}
and \eqref{dis_out} together, then we have receiver distribution as
\begin{equation}
\begin{aligned}
&p_R(y^{(m)}|\boldsymbol{\theta}^{(m)};t,\alpha)\\
&=\sum_{m=1}^M p_O(m|\boldsymbol{\theta}^{(m)};t) \mathcal{N} (\alpha (M\mathbf{e}_{m}-\mathbf{1}),\alpha M \mathbf{I}).
\end{aligned}
\end{equation}

\subsection{Joint Noise Strategy}
Under the assumption of independence between timestamps and event types, the joint sender distribution for modeling marked temporal data can be written as follows:
\begin{equation}
\begin{aligned}
&p_S\left(\left(\begin{array}{c}
         y^{(\tau)}  \\
         y^{(m)}  
    \end{array}\right)\g \mathbf{g}, \alpha\right)\\
    &=\mathcal{N}\left(\left(\begin{array}{c}
         y^{(\tau)}  \\
         y^{(m)}  
    \end{array}\right) \g \left(\begin{array}{c}
         \tau  \\
         \alpha(M\mbe_{m}-\mathbf{1}) 
    \end{array}\right), \left(\begin{array}{cc}
         {\alpha^{-1}} & \mathbf{0}^{\top} \\
         \mathbf{0} & \alpha M\mbI
    \end{array}\right)\right).
\end{aligned}
\end{equation}

However, timestamps and event types usually show complex coupling correlations in MTPP tasks. In this paper, we propose a joint noise strategy to capture their interdependence, improving the model's ability to learn the marked temporal joint distribution. To be specific, we add the joint noise to the marked temporal data space, the modified joint sender distribution can be defined as
\begin{equation}
\label{cov_mat}
\begin{aligned}
&p_S\left(\left(\begin{array}{c}
         y^{(\tau)}  \\
         y^{(m)}  
    \end{array}\right)\g \mathbf{g}, \alpha\right)\\
    &=\mathcal{N}\left(\left(\begin{array}{c}
         y^{(\tau)}  \\
         y^{(m)}  
    \end{array}\right) \g \left(\begin{array}{c}
         \tau  \\
         \alpha(M\mbe_{m}-\mathbf{1}) 
    \end{array}\right), \left(\begin{array}{cc}
         {\alpha^{-1}} & \mbc^{\top} \\
         \mbc & \alpha M\mbI
    \end{array}\right)\right),
\end{aligned}
\end{equation}
where the vector $\mathbf{c}\in \mathbb{R}^{M \times 1}$ denotes the covariance of time interval and different event types, which explicitly reveals the intricate interdependence
between timestamps and event types.

Let $\mbSigma_{{\tau}, m}=\left(\begin{array}{cc}
         {\alpha^{-1}} & \mbc^{\top} \\
         \mbc & \alpha M\mbI
    \end{array}\right)$ represent the covariance matrix in \eqref{cov_mat}, we have the conditions for the value of $\mathbf{c}$ as follows:
\begin{proposition}
To make sure the matrix $\mbSigma_{\tau, m}$ is positive definite, $\mbc$ is restricted as $\mbc^{\top}\mbc<M$.     
\end{proposition}

\begin{proof}
To prove matrix $\mbSigma_{\tau, m}$ is positive definite, we will apply Sylvester's criterion \cite{gilbert1991positive}, which states that a Hermitian matrix is positive definite if all its leading principal minors are positive.

The first leading principal minor is given by $\alpha^{-1}$. For $\mbSigma_{\tau, m}$ to be positive definite, we require $\alpha>0$.  The second leading principal minor is the determinant of $\mbSigma_{{\tau}, m}$, which can be given by
\begin{equation}
\det(\mbSigma_{{\tau}, m})=\alpha^{-1} \alpha M-\mbc^{\top}\mbc.
\end{equation}
To ensure that this determinant is positive, we need
\begin{equation}
\alpha^{-1} \alpha M-\mbc^{\top}\mbc>0 \Rightarrow \mbc^{\top}\mbc<M.
\end{equation}

\end{proof}

Similar to conditional diffusion models \cite{rombach2022high,lin2022exploring}, the input to the neural network in BMTPP includes not only the joint distribution parameters $\boldsymbol{\theta}=\{\boldsymbol{\theta}^{(\tau)},\boldsymbol{\theta}^{(m)}\}$ and the time step $t$, but also the historical embedding $\mathbf{h}_{i-1}$ as the condition, which provides information about historical events to predict future events. Consequently, the joint receiver distribution can be expressed as follows:

\begin{equation}
\begin{aligned}
&p_R\left(\left(\begin{array}{c}
         y^{(\tau)}  \\
         y^{(m)}  
    \end{array}\right)\g \boldsymbol{\theta}; \mathbf{h}_{i-1}, \alpha, t\right)\\
&=\sum_{m=1}^Mp_O(m|\Psi({\mbtheta}, \mathbf{h}_{i-1},t))\\
    &\mathcal{N}\left(\left(\begin{array}{c}
         y^{(\tau)}  \\
         y^{(m)}  
    \end{array}\right)|\left(\begin{array}{c}
         \Psi({\mbtheta},\mathbf{h}_{i-1},t)  \\
         \alpha(M\mbe_{m}-1) 
    \end{array}\right), \left(\begin{array}{cc}
         {\alpha^{-1}} & \mbc^{\top} \\
         \mbc & \alpha M\mbI
    \end{array}\right)\right),
\end{aligned}
\end{equation}
where we employ commonly used attention-based models to generate the $D$-dimensional historical embedding $\mathbf{h}_{i-1}$ as in \cite{lin2022exploring}.

The new objective function can be obtained by calculating the KL divergence between the joint sender distribution and joint receiver distribution:


\begin{equation}
\label{loss}
\begin{aligned}
&L^{\infty}(\mathbf{g}) \\
&= \lim_{\epsilon\to0}\frac{1}{\epsilon}\mathbb{E}_{t\sim U(\epsilon,1),p_F(\boldsymbol{\theta}|\mathbf{g},t-\epsilon)}
    \\
    &\bigg[\ln p_S\left(\left(\begin{array}{c}
         y^{{(\tau)}}  \\
         y^{(m)}  
    \end{array}\right)\g \mathbf{g}, \alpha\right) -\ln \sum_{m=1}^Mp_O(m|\Psi({\mbtheta}, \mathbf{h}_{i-1},t))\\
    &\mathcal{N}\left(\left(\begin{array}{c}
         y^{(\tau)}  \\
         y^{(m)}  
    \end{array}\right)|\left(\begin{array}{c}
         \Psi({\mbtheta},\mathbf{h}_{i-1},t)  \\
         \alpha(M\mbe_{m}-1) 
    \end{array}\right), \left(\begin{array}{cc}
         {\alpha^{-1}} & \mbc^{\top} \\
         \mbc & \alpha M\mbI
    \end{array}\right)\right)
\bigg].
\end{aligned}
\end{equation}

\subsection{Algorithm}
\textbf{Training Process.} Given the prior distribution of marked temporal data $\boldsymbol{\theta}_0=\{\boldsymbol{\theta}_0^{(\tau)},\boldsymbol{\theta}_0^{(m)} \}$, accuracy schedule $\beta(t)$, and historical embedding $\mathbf{h}_{i-1}$. We can train the BMTPP with the continuous-time loss \eqref{loss} in a recursive manner.

\textbf{Sampling Process.} With trained BMTPP, we can obtain the samples from the output distribution $p_O(\mathbf{g}|\boldsymbol{\theta};\mathbf{h}_{i-1},t)$ conditioned on the historical embedding $\mathbf{h}_{i-1}$. 

To provide a clearer understanding of the training and sampling process for the proposed approach, we have outlined the detailed algorithms in Algorithm \ref{alg1}, Algorithm \ref{alg2}, and Algorithm \ref{alg3}. Fig. \ref{fig:BMTPP} illustrates the entire framework of BMTPP.

\begin{figure*}[t]
\centering
\includegraphics[width=0.85\linewidth]{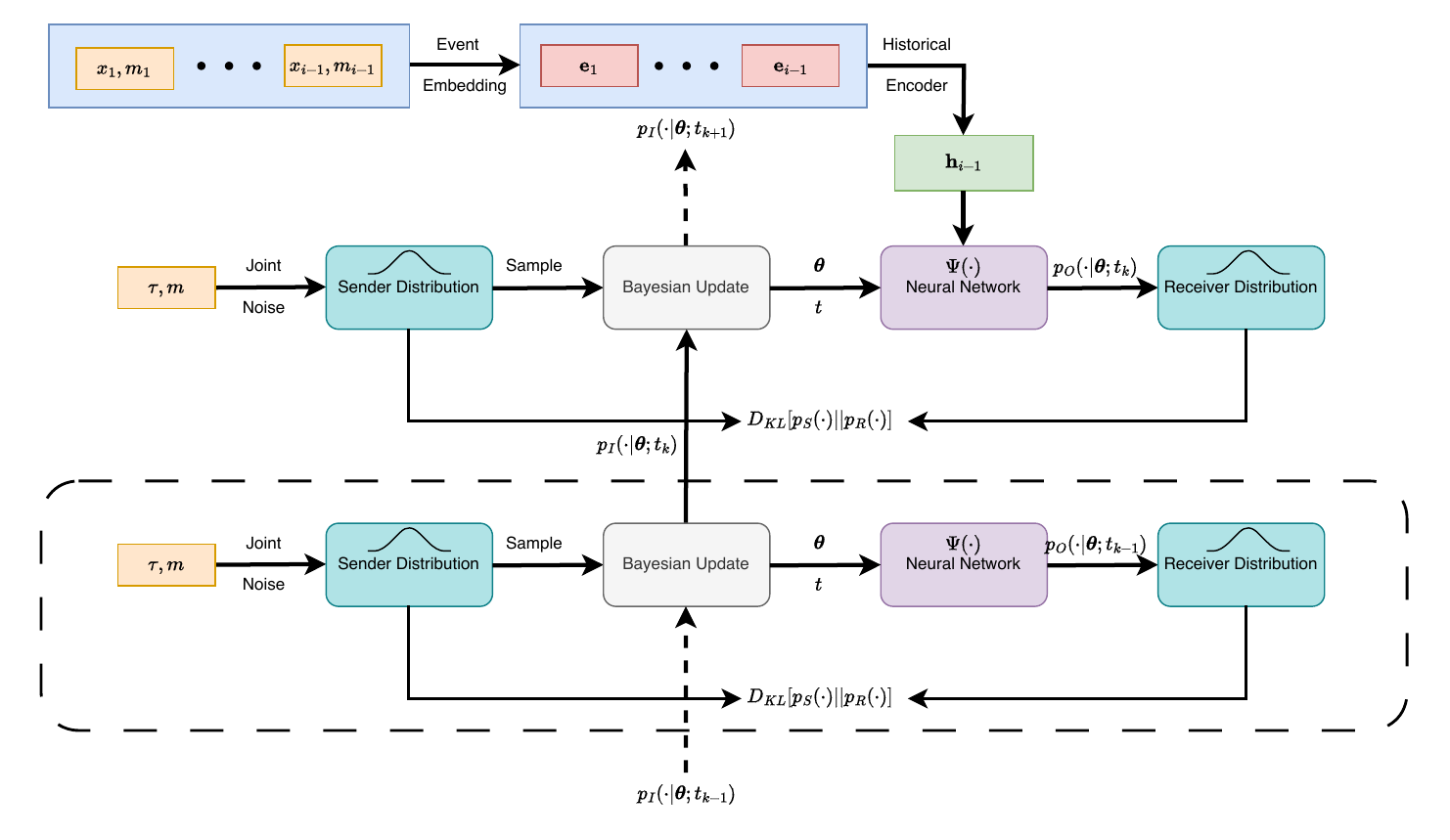}
\caption{The framework of BMTPP. We aim to learn the marked temporal joint distribution conditioned on the historical embedding $\mathbf{h}_{i-1}$ generated by the historical encoder.}
\label{fig:BMTPP}
\end{figure*}

\begin{algorithm}
	\renewcommand{\algorithmicrequire}{\textbf{function:} \parbox[t]{.8\textwidth}{OUTPUT\textunderscore PREDICTION($\mu$, $\boldsymbol{\theta}^{(m)} \in \mathbb{R}^{M\times1}$, $t\in[0,1]$, \\
    $t_{\min}\in \mathbb{R}^{+}$, $x_{\min}$, $x_{\max} \in \mathbb{R}$, $\mathbf{h}_{i-1}\in\mathbb{R}^D$)}}
	\renewcommand{\algorithmicensure}{\textbf{end function}}
	\caption{Function for BMTPP}
	\label{alg1}
	\begin{algorithmic}
        \coloredstate{\# Note that $\boldsymbol{\theta}^{(\tau)}=\{\mu,\rho\}$, but $\rho$ is fully determined by $t$}
        \REQUIRE
        \IF{$t<t_{\min}$}
        \coloredstate{\# In our experiment, we set $t_{\min}=1e-6$ and  $[x_{\min}, x_{\max}]=[-1,1]$}
        \STATE $\hat{\tau} \leftarrow 0$
        \ELSE
        \STATE Input $(\mu^{(\tau)}, \boldsymbol{\theta}^{(m)}, t, \mathbf{h}_{i-1})$ to the neural network, receive $\hat{\tau}$, $\Psi(\boldsymbol{\theta},t, \mathbf{h}_{i-1})$ as output
        \STATE clip $\hat{\tau}$ to $x_{\min}, x_{\max}$
        \STATE $p_O(m|\boldsymbol{\theta};t) \leftarrow \mathrm{softmax}(\Psi(\boldsymbol{\theta},t,\mathbf{h}_{i-1}))$
        \ENDIF
        \RETURN $\hat{\tau}$, $p_O(m|\boldsymbol{\theta};t)$
        \ENSURE
	\end{algorithmic}  
\end{algorithm}

\begin{algorithm}
	\renewcommand{\algorithmicensure}{\textbf{end function}}
	\caption{Training Process with Continuous-Time Loss}
	\label{alg2}
	\begin{algorithmic}
        \STATE \textbf{Require:} $\sigma_1 \in \mathbb{R}^+, \beta(1)\in \mathbb{R}^+, M \in \mathbb{N}, \mathbf{h}_{i-1} \in \mathbb{R}^D$
        \STATE \textbf{Input:} timestamp $\tau$ and mark $m$
        \STATE $t \sim U(0,1)$
        \STATE $\gamma \leftarrow 1-\sigma_1^{2t}$
        \STATE $\mu \sim \mathcal{N}(\gamma \tau, \tau(1-\gamma))$
        \STATE $\beta \leftarrow \beta(1)t^2$
        \STATE $\hat{\tau}, p_O(m|\boldsymbol{\theta};t) \leftarrow$ OUTPUT\textunderscore PREDICTION($\boldsymbol{\theta}, \mathbf{h}_{i-1},t$)
        \STATE $L^{\infty}(\mathbf{g}) \leftarrow f_{\text{OBJ}}(\hat{\tau}, p_O(m|\boldsymbol{\theta};\mathbf{h}_{i-1},t)) $ 
        \RETURN $L^{\infty}(\mathbf{g})$
	\end{algorithmic}  
\end{algorithm}

\begin{algorithm}
	\renewcommand{\algorithmicensure}{\textbf{end function}}
	\caption{Sample Generation}
	\label{alg3}
	\begin{algorithmic}
        \STATE \textbf{Require:} $\sigma_1 \in \mathbb{R}^+, \beta(1)\in \mathbb{R}^+, M \in \mathbb{N}$, number of steps $K\in \mathbb{N}$, $\mathbf{h}_{i-1}\in \mathbb{R}^D$
        \STATE $\mu \leftarrow 0$
        \STATE $\rho \leftarrow 1$
        \STATE $\boldsymbol{\theta}^{(m)} \leftarrow (\frac{\mathbf{1}}{\mathbf{M}})$
        \FOR{$k=1$ to $K$}
        \STATE $t \leftarrow \frac{k-1}{K}$
        \STATE $\hat{\tau}, p_O(m|\boldsymbol{\theta};t) \leftarrow$ OUTPUT\textunderscore PREDICTION($\boldsymbol{\theta}, \mathbf{h}_{i-1},t$)
        \STATE $\alpha \leftarrow \sigma_1^{-2i/K}(1-\sigma_1^{2/K})$
        \STATE $y^{(\tau)} \sim \mathcal{N}(\hat{\tau},\alpha^{-1})$
        \STATE $y^{(m)} \sim \mathcal{N}(\alpha M(\mathbf{e}_m-\mathbf{1},\alpha M \mathbf{I}))$
        \ENDFOR
        \STATE $\hat{\tau}, p_O(m|\boldsymbol{\theta};1) \leftarrow$ OUTPUT\textunderscore PREDICTION($\boldsymbol{\theta}, \mathbf{h}_{i-1},1$)
        \RETURN $\hat{\tau},  p_O(m|\boldsymbol{\theta};1)$
	\end{algorithmic}  
\end{algorithm}

\section{Experiments}
\subsection{Experimental Setup}

\textbf{Datasets.} We evaluate our model on four popular real-world datasets containing event data from a wide range of fields. Table \ref{tab:statistics} summarizes the statistics of each real-world dataset.
\begin{table}[t]
    \centering
    \caption{The statistics of four real-world datasets.}
    \label{tab:statistics}
    \begin{tabular}{cccccc}
    \toprule
  \multirow{2}{*}{Dataset}  & \multirow{2}{*}{$\#$ of sequences} & \multirow{2}{*}{Types-$M$} & \multicolumn{3}{c}{Sequence Length}  \\
   \cmidrule{4-6} & & &  Min & Mean & Max  \\
   \midrule
MOOC & $7047$ & $97$ & $4$ & $56$ & $493$ \\
Retweet & $24000$ & $3$ & $50$ & $109$ & $264$\\
Reddit & $10000$ & $984$ & $46$ & $68$ & $100$\\
Stack Overflow & $6633$ & $22$ & $41$ & $72$ & $736$\\
    \bottomrule
      \end{tabular}
\end{table} 
\begin{enumerate}
    \item \textit{MOOC} \cite{song2024decoupled}. This dataset contains event sequences of user interaction behaviors with the online course system. Each user has a sequence of interaction events and there are $M=97$ event types in total.
    \item \textit{Retweet} \cite{zuo2020transformer}. This dataset includes time-stamped user retweet event sequences. There are $M=3$ event types: retweets by "small", "medium" and "large" users. Users with 120 or fewer followers are small, those with between 121 and 1362 followers are medium, and any user with 1363 or more followers are large.
    \item \textit{Reddit} \cite{bae2023meta}. This dataset contains time-stamped user reddit event sequences. There are $M=984$ event types in total, which represent the sub-reddit categories of each sequence.
    \item \textit{Stack Overflow} \cite{zhang2020self}. This dataset includes the user awards on the Stack Overflow question-answering platform spanning a two-year period. Each user has a sequence of badge events and there are $M=22$ event types in total.
\end{enumerate}

\textbf{Baselines.} We compare BMTPP with eight state-of-the-art methods.
\begin{enumerate}
    \item \textit{RMTPP} \cite{du2016recurrent} is a neural MTPP model, which uses an RNN as the historical encoder and the mixture of Gompertz distribution as the intensity function.
    \item \textit{SAHP} \cite{zhang2020self} is a neural MTPP model, which uses the Transformer as the historical encoder, and model the intensity function with an exponential-decayed formulation.
    \item \textit{THP} \cite{zuo2020transformer} is a neural MTPP model, which employs the Transformer as the historical encoder, and models the intensity function with a linear function formulation.
    \item \textit{Log-Norm} \cite{shchur2019intensity} is a neural MTPP model, which uses the LSTM as the historical encoder and the Log-Normal distribution as the intensity function.
    \item \textit{TCVAE} \cite{lin2022exploring} is a neural MTPP model, which employs a revised Transformer as the historical encoder and the VAE as the probabilistic decoder.
    \item \textit{TCGAN} \cite{lin2022exploring} is a generative MTPP model, which uses a revised Transformer as the historical encoder and the GAN as the probabilistic decoder.
    \item \textit{TCDDM} \cite{lin2022exploring} is a generative MTPP model, which utilizes a revised Transformer as the historical encoder and the diffusion model as the probabilistic decoder.
    \item \textit{Dec-ODE} \cite{song2024decoupled} is a recently proposed MTPP model, which uses the neural ordinary differential equations (ODE) to learn the changing influence from individual events.
\end{enumerate}

\textbf{Training Protocol.} We have two important hyperparameters for each model. One is the \textit{embedding size}, i.e., the dimension of historical embedding, which we explore over $\{8,16,32\}$. The other is the \textit{layer number} in the encoder, which we choose from $\{1,2,3\}$. We determine these two hyperparameters for each model that produce the best results. The number of the training epochs and the learning rate are set to $200$ and $1e-4$, respectively, for all approaches. The sampling steps is fixed to $K=100$ for all experiments. In addition, we report the mean and the standard deviation (SD) for each metric based on three experimental results with three random seeds. We utilize the \textit{log-normalization} trick during the training process. To be specific, we first normalize the time intervals with $\frac{\log \tau-\mathrm{Mean} (\log \tau)}{\mathrm{Std}(\log \tau)}$, and then rescaled back through $\mathrm{exp}(\mathrm{Mean} (\log \tau)+\log \tau \cdot\mathrm{Std}(\log \tau))$ for ensuring the sampled time intervals are positive \cite{lin2022exploring}. All experiments are conducted on an NVIDIA Tesla V100 GPU with 60GB of memory.

\textbf{Metrics.} We evaluate our model with three different metrics.
\begin{enumerate}
    \item \textit{Mean absolute percent error} (MAPE) measures the predictive performance of the next time point \cite{lin2021empirical}.
    \item \textit{Continuous ranked probability score} (CRPS) is employed for measuring the goodness-of-fit, as there is no closed form of the conventional metric negative log-likelihood (NLL) for deep generative models \cite{taieb2022learning}.
    
    Follow the definition in \cite{taieb2022learning}, the CRPS can be represented as
    \begin{equation}
    \operatorname{CRPS}\left(F_{\boldsymbol{\theta}}, \tau\right)=\mathbb{E}_{F_{\boldsymbol{\theta}}}\left|Z_1-\tau\right|-\frac{1}{2} \mathbb{E}_{F_{\boldsymbol{\theta}}, F_{\boldsymbol{\theta}}}\left|Z_1-Z_2\right|,
    \end{equation}
    where $F_{\boldsymbol{\theta}}$ denotes the predictive cumulative distribution function with finite first moment; $Z_1$ and $Z_2$ are two independent random variables following the distribution $F_{\boldsymbol{\theta}}$ \cite{jordan2017evaluating,taieb2022learning}.
    
    Since there is no closed form of CRPS for all compared models, we utilize the sampling-based method to calculate CRPS. We can first approximate the $F_{\boldsymbol{\theta}}$ with the empirical CDF as
    \begin{equation}
    F_{\boldsymbol{\theta}}^{(l)}(\tau)=\frac{1}{L} \sum_{l=1}^L \mathbf{1} \left\{\tau_i \leq \tau\right\},
    \end{equation}
    where $l$ stands for the number of samples drawn from the distribution $F_{\boldsymbol{\theta}}$ and $\mathbf{1} \left\{\tau_i \leq \tau\right\}$ is the indicator function. We can then calculate the CRPS as follows:
    \begin{equation}
    \label{eq:crps}
    \operatorname{CRPS}\left(F_{\boldsymbol{\theta}}^{(l)}, \tau\right)=\frac{1}{L} \sum_{l=1}^L\left|\tau_l-\tau\right|-\frac{1}{2 L^2} \sum_{l=1}^L \sum_{g=1}^L\left|\tau_l-\tau_g\right| .
    \end{equation}
    Here, $F_{\boldsymbol{\theta}}$ represents the $p_{\theta}(\tau|\mbh_{i-1})$ in this paper, and $\tau$ is the ground truth observation. According to the \eqref{eq:crps}, We can see that CRPS takes into account both the sampling quality and sampling diversity \cite{lin2022exploring}. 
    \item \textit{Accuracy} (ACC) validates the ability of models to predict the next event type \cite{zhang2020self}.
\end{enumerate}

\subsection{Performance Comparison}
\begin{table*}[t]
\renewcommand{\arraystretch}{1.0}
\caption{Comparison results (± SD) of different methods on four real-world datasets (Results with \textbf{Bold} and \underline{underline} represent the best and the second best results, respectively).}
\label{tab:comparison}
\centering

\resizebox{\linewidth}{!}{

\begin{tabular}{c|c|cccccccc|c}
\toprule
 Dataset & Metric & RMTPP & SAHP & THP & LogNorm & TCVAE & TCGAN & TCDDM & Dec-ODE & BMTPP\\
 \midrule
 \multirow{3}{*}{MOOC} & MAPE $\downarrow$ & 87.674\scriptsize±0.331 & 52.623\scriptsize±0.783 & 46.621\scriptsize±0.874 & 94.234\scriptsize±0.365 & $\boldsymbol{22.564}$\scriptsize±2.431 & 26.269\scriptsize±2.314 & 24.256\scriptsize±0.247 & 43.376\scriptsize±0.423 & $\underline{22.753}$\scriptsize±0.263\\
 & CRPS $\downarrow$ & 32.670\scriptsize±0.426 & $-$ & $-$ & 31.365\scriptsize±0.114 & $\boldsymbol{0.147}$\scriptsize±0.004 & 1.658\scriptsize±0.045 & $\underline{1.137}$\scriptsize±0.001 & $-$ & 1.571\scriptsize±0.026\\
 & ACC $\uparrow$ & 0.281\scriptsize±0.046 & 0.301\scriptsize±0.004 & 0.313\scriptsize±0.015 & 0.324\scriptsize±0.007 & 0.314\scriptsize±0.012 & $\underline{0.333}$\scriptsize±0.024 & 0.321\scriptsize±0.013 & 0.317\scriptsize±0.025 & $\boldsymbol{0.483}$\scriptsize±0.057\\
 \midrule
  \multirow{3}{*}{Retweet} & MAPE $\downarrow$ & 70.270\scriptsize±2.931 & 26.604\scriptsize±0.087 & $\boldsymbol{16.631}$\scriptsize±0.092 & 95.306\scriptsize±0.043 & 21.099\scriptsize±0.038 & 26.357\scriptsize±2.464 & 16.657\scriptsize±1.373 & 28.973\scriptsize±0.056 & $\underline{18.379}$\scriptsize±0.045\\
 & CRPS $\downarrow$ & 0.445\scriptsize±0.091 & $-$ & $-$ & 0.443\scriptsize±0.007 & $\boldsymbol{0.154}$\scriptsize±0.007 & 0.563\scriptsize±0.018 & 0.438\scriptsize±0.010 & $-$ & $\underline{0.325}$\scriptsize±0.043\\
 & ACC $\uparrow$ & $\underline{0.605}$\scriptsize±0.002 & 0.583\scriptsize±0.014 & 0.586\scriptsize±0.012 & 0.604\scriptsize±0.004 & 0.563\scriptsize±0.045 & $\underline{0.605}$\scriptsize±0.018 & 0.589\scriptsize±0.021 & 0.588\scriptsize±0.028 & $\boldsymbol{0.685}$\scriptsize±0.045\\
  \midrule
  \multirow{3}{*}{Reddit} & MAPE $\downarrow$ & 20.417\scriptsize±0.423 & 23.795\scriptsize±0.124 & 26.833\scriptsize±0.189 & 81.171\scriptsize±0.324 & 23.678\scriptsize±0.125 & 26.906\scriptsize±0.337 & $\underline{21.293}$\scriptsize±0.287 & 23.665\scriptsize±0.173 & $\boldsymbol{19.247}$\scriptsize±0.361\\
 & CRPS $\downarrow$ & 35.213\scriptsize±0.527 & $-$ & $-$ & 34.050\scriptsize±0.148 & $\boldsymbol{0.407}$\scriptsize±0.023 & 0.837\scriptsize±0.034 & 0.902\scriptsize±0.009 & $-$ & $\underline{0.754}$\scriptsize±0.072\\
 & ACC $\uparrow$ & 0.494\scriptsize±0.026 & 0.482\scriptsize±0.034 & 0.487\scriptsize±0.016 & 0.470\scriptsize±0.023 & 0.537\scriptsize±0.007 & $\underline{0.549}$\scriptsize±0.013 & 0.539\scriptsize±0.024 & 0.493\scriptsize±0.015 & $\boldsymbol{0.604}$\scriptsize±0.037\\
  \midrule
  \multirow{3}{*}{\makecell{Stack \\ Overflow}} & MAPE $\downarrow$ & 6.557\scriptsize±1.243 & 5.364\scriptsize±0.038 & 5.631\scriptsize±0.046 & 17.087\scriptsize±1.163 & 6.028\scriptsize±0.053 & $\underline{4.458}$\scriptsize±0.067 & 4.931\scriptsize±0.041 & 5.474\scriptsize±0.055 & $\boldsymbol{3.527}$\scriptsize±0.078\\
 & CRPS $\downarrow$ & 6.179\scriptsize±0.122 & $-$ & $-$ & 6.308\scriptsize±0.0.257 & $\boldsymbol{0.397}$\scriptsize±0.012 & 0.474\scriptsize±0.028 & $\underline{0.454}$\scriptsize±0.014 & $-$ & 0.463\scriptsize±0.088\\
 & ACC $\uparrow$ & \underline{0.529}\scriptsize±0.024 & 0.528\scriptsize±0.021 & 0.528\scriptsize±0.034 & $\underline{0.529}$\scriptsize±0.022 & 0.528\scriptsize±0.004 & 0.527\scriptsize±0.014 & 0.528\scriptsize±0.025 & 0.528\scriptsize±0.028 & $\boldsymbol{0.615}$\scriptsize±0.014\\
\bottomrule
\end{tabular}}

\end{table*}
We compare BMTPP with state-of-the-art MTPP models in four real-world datasets. Notice that since calculating CRPS requires flexible sampling, which is prohibited in SAHP, THP and Dec-ODE, we do not report the CRPS for them \cite{lin2022exploring}. The main results are presented in Table \ref{tab:comparison}, from which we can summarize the following conclusions:
\begin{enumerate}
    \item For the MAPE metric, we see that BMTPP achieved either the best or second-best performance across all datasets, demonstrating the superiority of our method in predicting future time points compared to other algorithms. Additionally, we observe that generative MTPP models generally outperformed classical MTPP models on all datasets, particularly on the MOOC and Retweet datasets. This further validates that flexibly modeling the likelihood function contributes to improved model performance.
    \item For the CRPS metric, we find that generative MTPP models consistently outperformed classical MTPP models across all datasets, highlighting their superior model-fitting capabilities. Compared to other baseline methods, BMTPP exhibits competitive performance on MOOC and Retweet datasets, as demonstrated by achieving the second-best results on both. We also observed that TCVAE dominates the other approaches in terms of CRPS thanks to the VAE's powerful ability to capture complex distributions and latent representations.
    \item For the ACC metric, we observe that BMTPP achieves the best performance across all datasets. This is mainly because classical MTPP models, generative MTPP models, and even the recent Dec-ODE model rely on model assumptions or numerical approximation techniques to predict future event types. Thanks to the flexible modeling of discrete data by BFN, BMTPP is able to predict the next event type without being constrained by these limitations.
\end{enumerate}

\subsection{Analysis of Vector $\mathbf{c}$}
After completing the training of the BMTPP model, we can obtain the vector $\mathbf{c}$ in \eqref{cov_mat}, which denotes the covariance of time interval and different event types and also reflects the interdependence between timestamps and event types. To explicitly reveal this interaction, we present the number of elements from vector $\mathbf{c}$ across different value ranges in Fig. \ref{fig:c_value}. The results show that most of the element values are concentrated around zero, indicating that the correlation between the timestamp and most event types is weak. However, a few event types exhibit a strong correlation with the timestamp. This aligns with real-world situations; for example, during winter, discussions about sweaters significantly exceed those about other types of clothing. Furthermore, we observe that the interdependence of the timestamp and event types varies across different datasets, validating the complex interdependence between timestamps and event types in various scenarios. This also demonstrates that the proposed method can explicitly reveal such interdependence.

\begin{figure*}[t]
\centering
\includegraphics[width=0.24\linewidth]{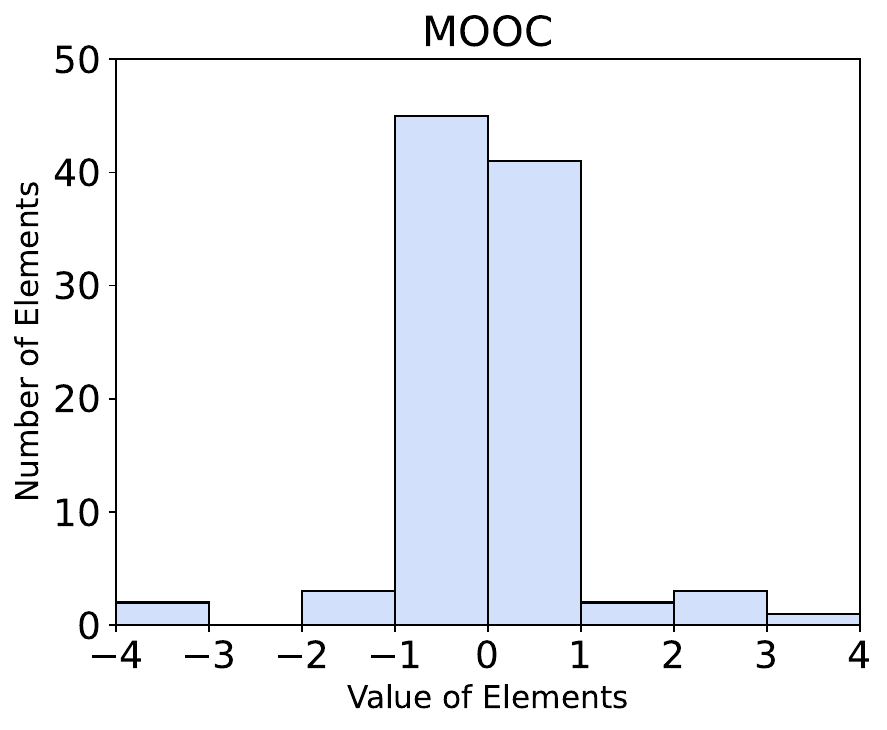}
\includegraphics[width=0.24\linewidth]{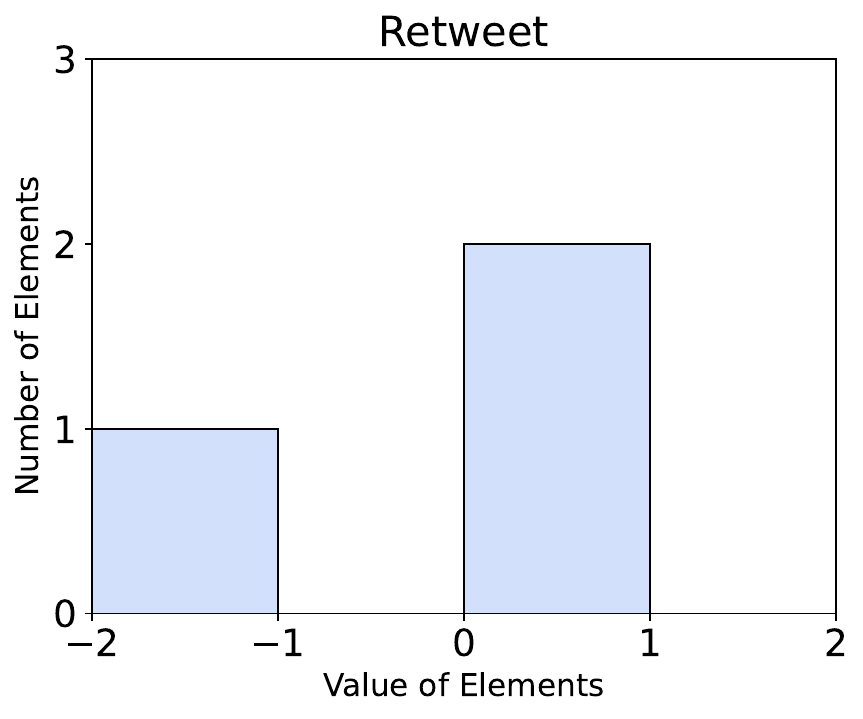}
\includegraphics[width=0.24\linewidth]{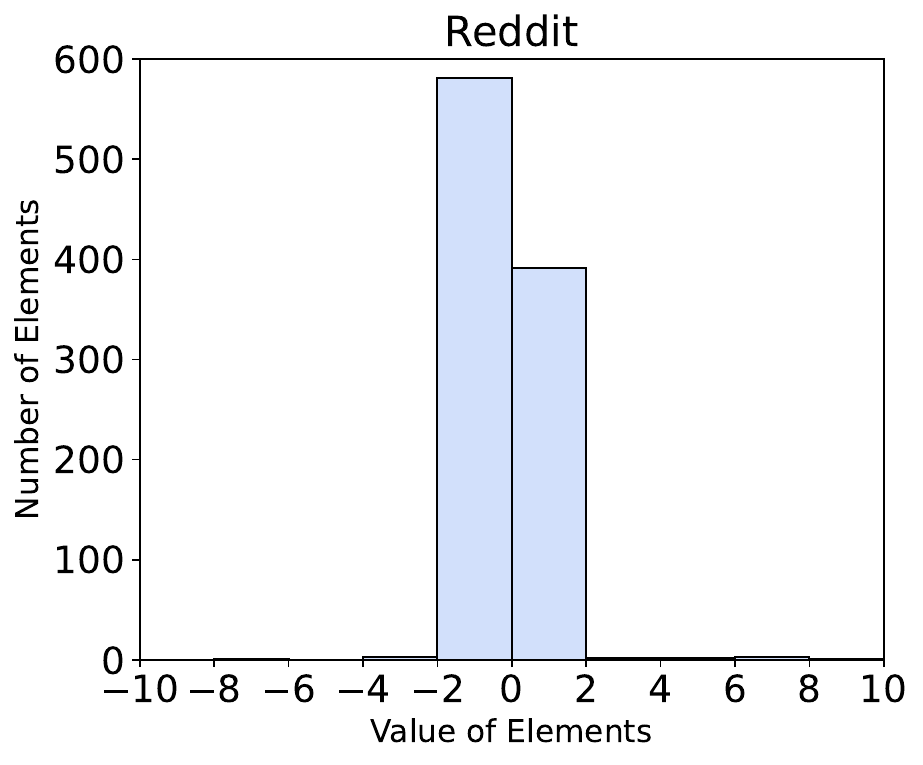}
\includegraphics[width=0.24\linewidth]{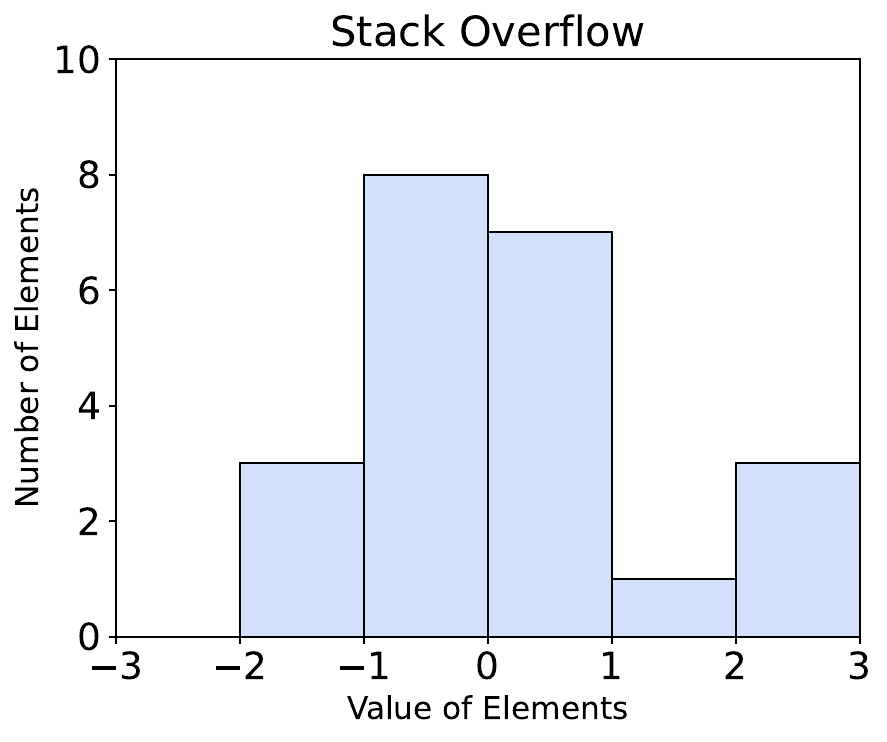}
\caption{The number of elements from vector $\mathbf{c}$ across different value ranges on four real-world datasets.}
\label{fig:c_value}
\end{figure*}

\subsection{Ablation Study}
BMTPP captures the interdependence between the timestamp and event types by incorporating the joint noise (JN) strategy. In this part, we conduct ablation studies to examine the effectiveness of this strategy. The main results are presented in Fig. \ref{fig:ablation}, from which we can see that the BMTPP with the joint noise strategy consistently outperforms the BMTPP without it across three metrics on all four real-world datasets. It implies that the interdependence between the timestamp and event types is closely related to model performance. Our joint noise strategy effectively captures this interdependence, resulting in significant improvements in model performance.

\begin{figure*}[t]
\centering
\includegraphics[width=0.32\linewidth]{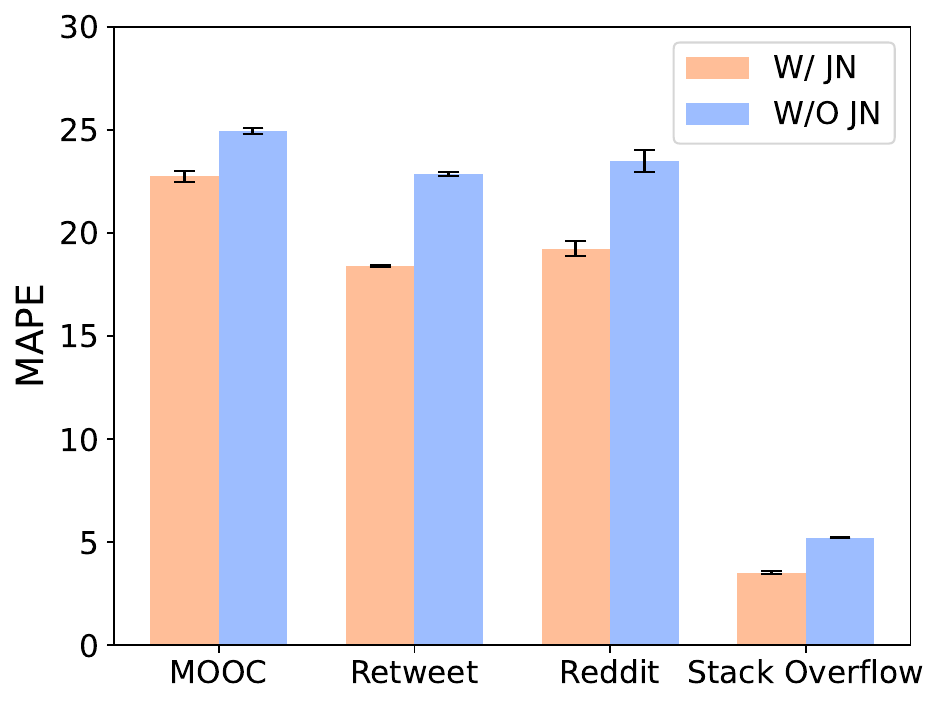}
\includegraphics[width=0.32\linewidth]{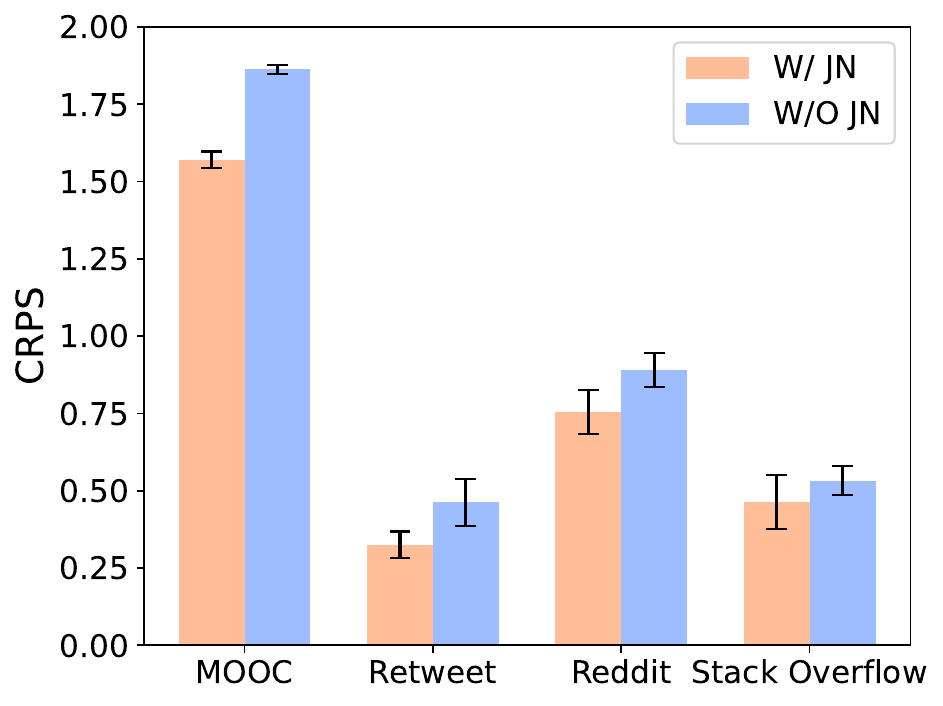}
\includegraphics[width=0.32\linewidth]{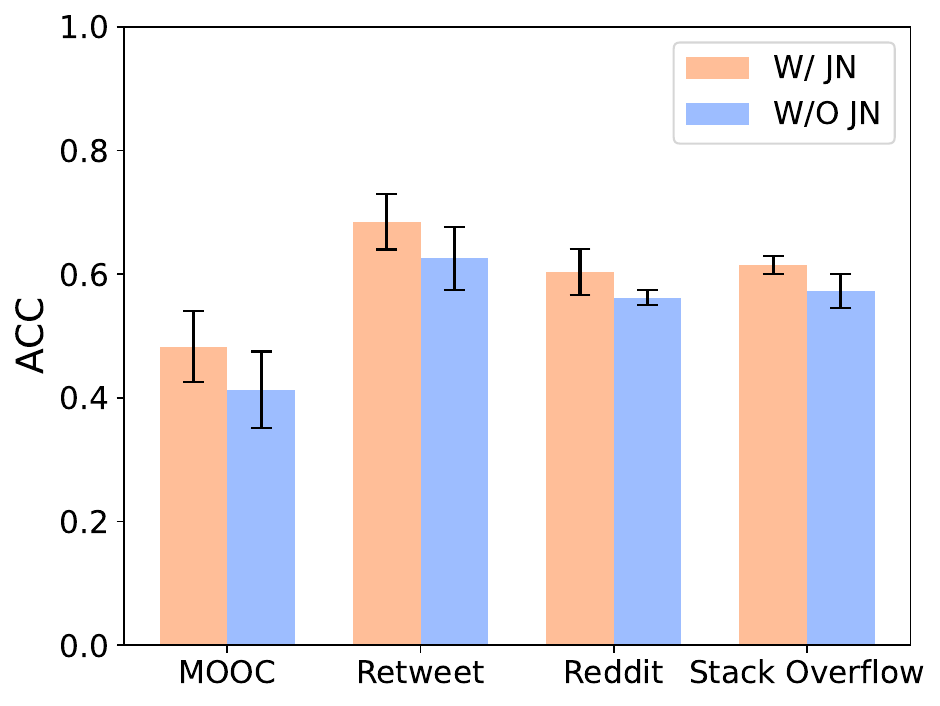}
\caption{Ablation study. W/ JN and W/O JN represent the model with and without the joint noise strategy, respectively.}
\label{fig:ablation}
\end{figure*}

\subsection{Hyperparameter Sensitivity Analysis}
\begin{table*}[t]
\renewcommand{\arraystretch}{1.0}
\caption{Performance comparison on four real-world datasets with different layer numbers and embedding size.}
\label{tab:hyperparameter}
\centering

\resizebox{\linewidth}{!}{

\begin{tabular}{c|c|ccc|ccc}
\toprule
 \multirow{3}{*}{Dataset}& \multirow{3}{*}{Metric} & \multicolumn{3}{c|}{Layer Number} & \multicolumn{3}{c}{Embedding Size}\\
 \cmidrule(lr){3-5} \cmidrule(lr){6-8}
 & & 1 & 2 & 3 & 8 & 16 & 32\\
 \midrule
 \multirow{3}{*}{MOOC} & MAPE $\downarrow$ & $\boldsymbol{22.753}$\scriptsize±0.263 & 23.527\scriptsize±0.367 & 24.842\scriptsize±0.229 & 23.462\scriptsize±0.234 & $\boldsymbol{22.753}$\scriptsize±0.263 & 23.774\scriptsize±0.315\\
 & CRPS $\downarrow$ & $\boldsymbol{1.571}$\scriptsize±0.026 & 1.632\scriptsize±0.048 & 1.629\scriptsize±0.033 & 1.583\scriptsize±0.045 & 1.575\scriptsize±0.039 & $\boldsymbol{1.571}$\scriptsize±0.026\\
 & ACC $\uparrow$ & $\boldsymbol{0.483}$\scriptsize±0.057 & 0.476\scriptsize±0.068 & 0.474\scriptsize±0.082 & 0.481\scriptsize±0.053 & 0.479\scriptsize±0.032 & $\boldsymbol{0.483}$\scriptsize±0.057 \\
 \midrule
  \multirow{3}{*}{Retweet} & MAPE $\downarrow$ & 19.130\scriptsize±0.082 & $\boldsymbol{18.379}$\scriptsize±0.045 & 20.533\scriptsize±0.059 & 18.938\scriptsize±0.042 & 19.245\scriptsize±0.057 & $\boldsymbol{18.379}$\scriptsize±0.045\\
 & CRPS $\downarrow$ & $\boldsymbol{0.325}$\scriptsize±0.043 & 0.372\scriptsize±0.051 & 0.363\scriptsize±0.038 & 0.334\scriptsize±0.026 & 0.336\scriptsize±0.047 & $\boldsymbol{0.325}$\scriptsize±0.043\\
 & ACC $\uparrow$ & 0.681\scriptsize±0.036 & $\boldsymbol{0.685}$\scriptsize±0.045 & 0.675\scriptsize±0.037 & 0.678\scriptsize±0.051 & 0.680\scriptsize±0.044 & $\boldsymbol{0.685}$\scriptsize±0.045\\
  \midrule
  \multirow{3}{*}{Reddit} & MAPE $\downarrow$ & $\boldsymbol{19.247}$\scriptsize±0.361 & 20.202\scriptsize±0.317 & 19.339\scriptsize±0.389 & 20.592\scriptsize±0.412 & $\boldsymbol{19.247}$\scriptsize±0.361 & 21.154\scriptsize±0.366\\
 & CRPS $\downarrow$ & $\boldsymbol{0.754}$\scriptsize±0.072 & 0.762\scriptsize±0.068 & 0.779\scriptsize±0.073 & 0.761\scriptsize±0.061 & 0.758\scriptsize±0.088 & $\boldsymbol{0.754}$\scriptsize±0.072\\
 & ACC $\uparrow$ & $\boldsymbol{0.604}$\scriptsize±0.037 & 0.602\scriptsize±0.045 & 0.588\scriptsize±0.041 & 0.593\scriptsize±0.036 & $\boldsymbol{0.604}$\scriptsize±0.037 & 0.602\scriptsize±0.046\\
  \midrule
  \multirow{3}{*}{\makecell{Stack \\ Overflow}} & MAPE $\downarrow$ & $\boldsymbol{3.527}$\scriptsize±0.078 & 4.528\scriptsize±0.059 & 3.703\scriptsize±0.063 & 5.254\scriptsize±0.039 & $\boldsymbol{3.527}$\scriptsize±0.078 & 3.947\scriptsize±0.083\\
 & CRPS $\downarrow$ & 0.485\scriptsize±0.127 & $\boldsymbol{0.463}$\scriptsize±0.088 & 0.491\scriptsize±0.074 & 0.481\scriptsize±0.136 & 0.472\scriptsize±0.081 & $\boldsymbol{0.463}$\scriptsize±0.088\\
 & ACC $\uparrow$ & 0.609\scriptsize±0.022 & $\boldsymbol{0.615}$\scriptsize±0.014 & 0.613\scriptsize±0.028 & 0.594\scriptsize±0.029 & 0.607\scriptsize±0.015 & $\boldsymbol{0.615}$\scriptsize±0.014\\
\bottomrule
\end{tabular}}

\end{table*}
Observing the impact of hyperparameters on model performance is crucial. In this part, we conduct experiments to explore two important hyperparameters frequently investigated in the MTPP tasks \cite{yang2022transformer,lin2022exploring}: \textit{embedding size} and \textit{layer number}. More precisely, we set the embedding size and layer number over $\{8,16,32\}$ and $\{1,2,3\}$, respectively. The results for different metrics on four real-world datasets are recorded in Table \ref{tab:hyperparameter}. The results indicate that setting the layer number to 1 or 2 leads to better model performance, especially for the MOOC and Reddit datasets, where fewer layers result in improved performance. In contrast, we observe that setting the embedding size to 16 or 32 brings improvements in model performance, suggesting that the model benefits from larger embedding sizes.

\subsection{Effect of Sampling Steps}
Akin to the diffusion models, the number of sampling steps in BMTPP has a significant impact on model performance \cite{lin2023text}. Generally, a larger number of sampling steps can result in a closer alignment between the sender distribution and receiver distribution, leading to better model performance. However, too many sampling steps can hurt the model's sampling efficiency. Therefore, we need to choose an appropriate number of sampling steps to balance the model performance and efficiency. In this part, we conduct experiments on four real-world datasets with varying sampling steps over $\{5, 10, 25, 50, 100, 250, 500\}$. The results in Fig. \ref{fig:steps} show that when we set the number of sampling steps to 100, the model is able to achieve satisfactory performance while maintaining sampling efficiency.

\begin{figure*}[t]
\centering
\includegraphics[width=0.32\linewidth]{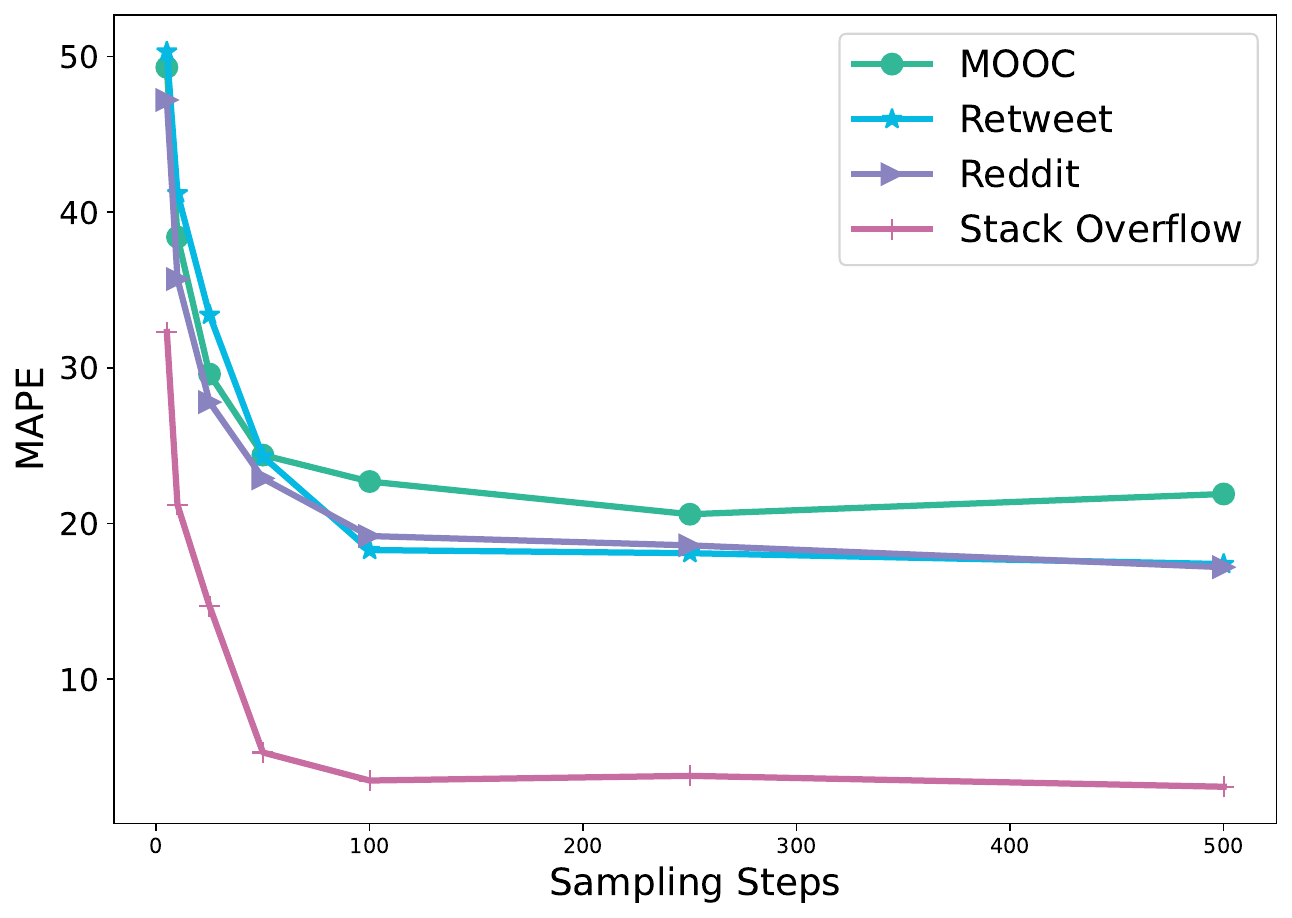}
\includegraphics[width=0.32\linewidth]{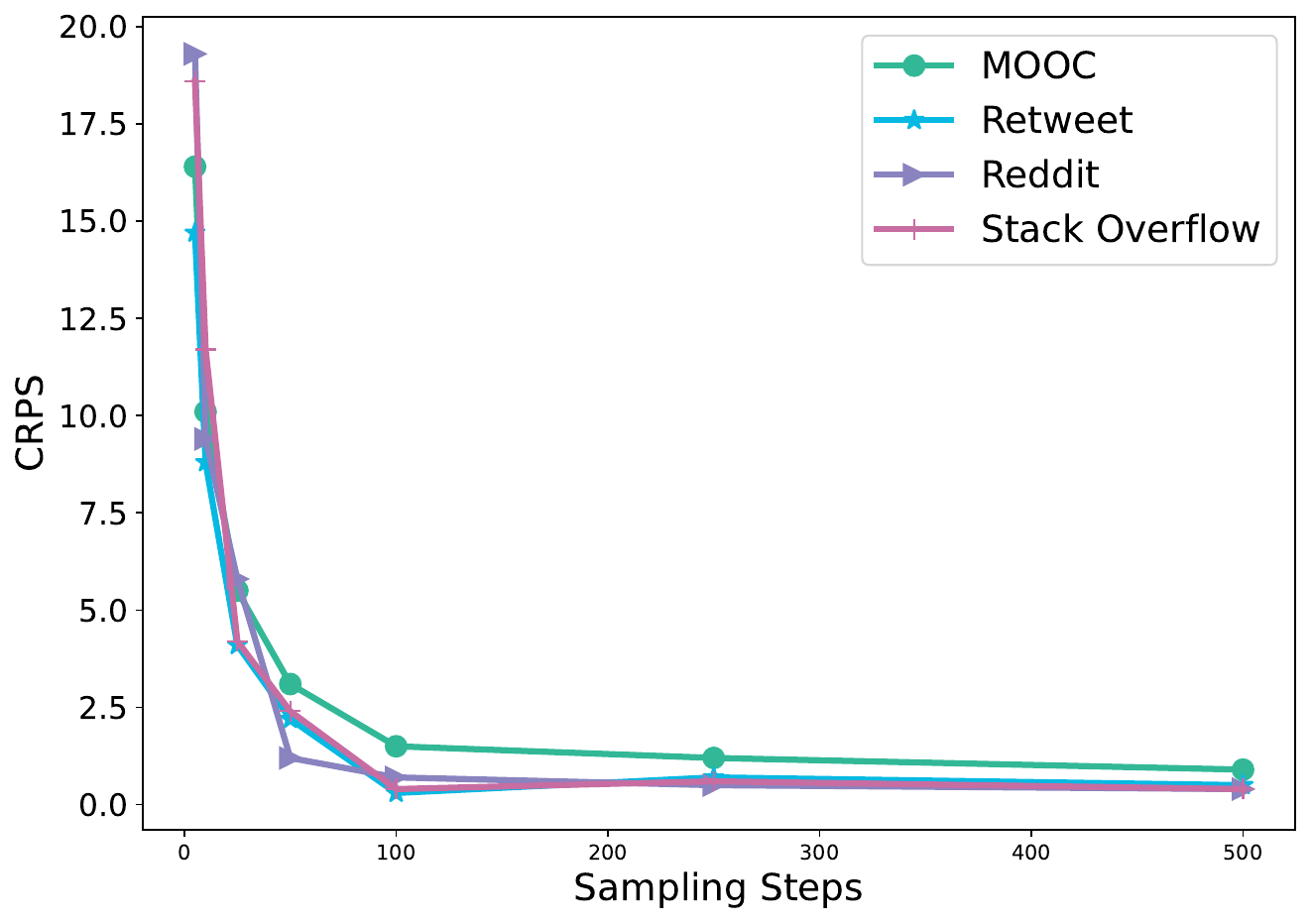}
\includegraphics[width=0.32\linewidth]{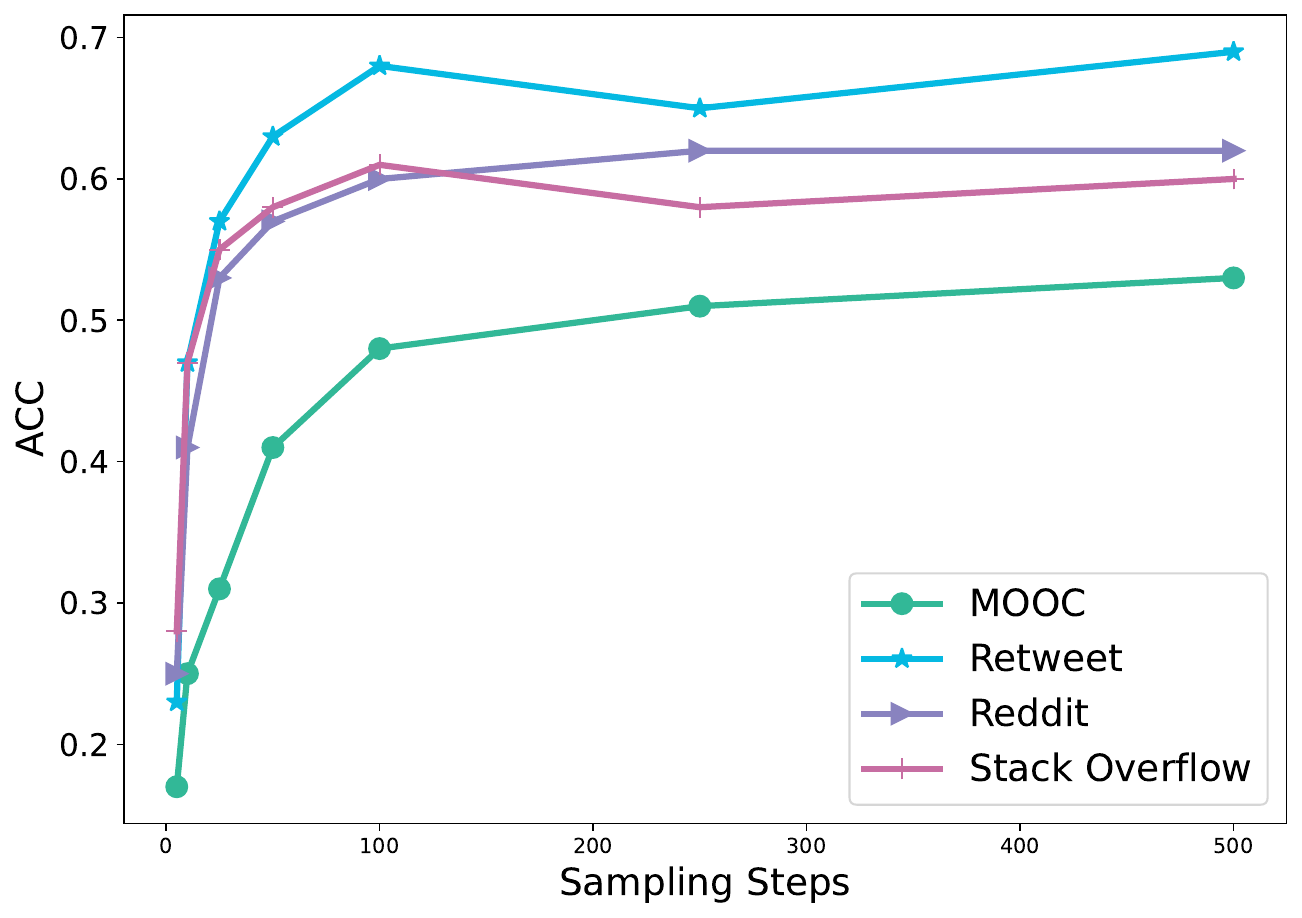}
\caption{Performance comparison on four real-world datasets with varying sampling steps.}
\label{fig:steps}
\end{figure*}

\subsection{Results on Negative Log-Likelihood}
\begin{table}[t]
\renewcommand{\arraystretch}{1.0}
\caption{The experimental results of negative log-likelihood on four real-world datasets.}
\label{tab:NLL}
\centering
\resizebox{\linewidth}{!}{
\begin{tabular}{c|cccc}
\toprule
 Dataset & MOOC & Retweet & Reddit & Stack Overflow\\
 \midrule
RMTPP      & 2.965\scriptsize±0.607 & -1.819\scriptsize±0.084 & 9.287\scriptsize±0.067 & 4.918\scriptsize±0.007 \\ 
SAHP       & -1.501\scriptsize±0.423 & -2.848\scriptsize±0.074 & 3.518\scriptsize±0.036  & 1.940\scriptsize±0.057 \\
THP        & 0.285\scriptsize±0.103 & -1.236\scriptsize±0.045 & 3.854\scriptsize±0.034 & 1.971\scriptsize±0.056 \\
LogNorm    & 1.931\scriptsize±0.284 &-2.155\scriptsize±0.094 & 9.369\scriptsize±0.063  & 4.897\scriptsize±0.013 \\ 
TCVAE      & $\leq$ 3.863\scriptsize±0.689 & $\leq$ 2.167\scriptsize±0.742 &  $\leq$ 4.408\scriptsize±0.058 & $\leq$ 2.593\scriptsize±0.473 \\
TCGAN      & (0.009\scriptsize±0.003\normalsize) & (0.016\scriptsize±0.007\normalsize) & (0.048\scriptsize±0.014\normalsize)  & (0.005\scriptsize±0.028\normalsize) \\ 
TCDDM      & $\leq$ 4.256\scriptsize±1.748 & $\leq$ 1.817\scriptsize±0.579 & $\leq$ 4.321\scriptsize±0.054  & $\leq$ 2.358\scriptsize±0.081 \\ 
Dec-ODE     & 0.374\scriptsize±0.533 & -1.736\scriptsize±0.062 & 3.645\scriptsize±0.057  & 2.443\scriptsize±0.025 \\ 
\midrule
BMTPP & $\leq$ 4.136\scriptsize±0.892 & $\leq$ 2.481\scriptsize±0.773 & $\leq$ 4.453\scriptsize±0.027& $\leq$ 2.588\scriptsize±0.093\\
\bottomrule
\end{tabular}}
\end{table}
Since obtaining the exact NLL for deep generative models is not feasible, we employ the CRPS to measure the goodness-of-fit in Table \ref{tab:comparison}. However, the exploration and comparison of the NLL remains meaningful. Thus, we follow the same evaluation configuration in prior works \cite{lin2022exploring,yuan2023spatio} and replace the NLL metric with negative variation lower bound (VLB), which is the upper bound of NLL for TCVAE, TCDDM, and BMTPP. Here, we calculate the negative VLB for BMTPP according to $L_{VLB}=L^{\infty}+L^r$, where the $L^r=-\mathbb{E}_{p_F(\boldsymbol{\theta}|\mathbf{g},1)}\ln p_O(\mathbf{g}|\boldsymbol{\theta},1)$ denotes the reconstruction losses\cite{graves2023bayesian}. For the TCGAN, we give the Wasserstein distance between empirical and model distribution in all four datasets \cite{lin2022exploring,xiao2017wasserstein}. The results in Table \ref{tab:NLL} show that the neural MTPP models using transformers as historical encoders, such as MTBB, SAHP, and THP, achieve strong performance across all datasets, indicating their strong data-fitting capabilities. Additionally, we find that two other neural MTPP models, RMTPP and LogNorm, perform poorly on most datasets, which is consistent with their results on the CRPS metric. Although generative models use VLB instead of NLL, their performance on some datasets is also competitive.

\section{Related Work}
\subsection{Temporal Point Processes}
In recent years, many neural TPP models have been proposed \cite{shchur2021neural}. Different from early classical TPP models, such as Poisson processes \cite{lin2021empirical}, Hawkes processes \cite{hawkes1971spectra}, and self-correcting processes \cite{isham1979self}, neural TPP models can learn complex dependencies among the events. \cite{du2016recurrent} firstly employed the RNN as the historical encoder for capturing these dependencies. To model a continuous time model, \cite{mei2017neural} used the LSTM as the historical encoder and designed a temporal memory cell for it. To better capture the long-range dependencies among the events, \cite{zuo2020transformer} and \cite{zhang2020self} introduced the attention-based historical encoder to improve the predictive performance. Aside from research on the historical encoder, another line of work focuses on studying the density decoder. \cite{omi2019fully} used the neural network to parametrize the intensity decoder of the cumulative hazard function. Moreover, the single Gaussian \cite{xiao2017modeling} and the mixture of log-norm \cite{shchur2019intensity} are also used for constructing the target distribution. More recently, generative models, including diffusion models, VAE, and GAN, have been investigated as density decoders due to their powerful generative capabilities on continuous data\cite{lin2022exploring}. Unlike these models, we utilize the BFN as the density decoder to flexibly handle both continuous and discrete data simultaneously.

\subsection{Bayesian Flow Networks}
Bayesian flow networks is a recently introduced deep generative model showing great promise in providing a uniform strategy for modeling continuous, discretized, and discrete observations \cite{graves2023bayesian}. Similar as diffusion models \cite{sohl2015deep,ho2020denoising}, BFN also requires multiple sampling steps to generate desired samples. The main difference is that diffusion models build the probability distributions over the sample space, while the BFN constructs over the space of distribution parameters. As a result, BFN can naturally address the challenges posed by discrete data in diffusion models \cite{zhou2024beta}. \cite{song2024unified} introduced the BFN to the 3D molecules for the first time to address the multi-modality and noise sensitivity issues. To explore the connection between the BFN and diffusion models, \cite{xue2024unifying} explained their relations from the perspective of stochastic differential equations. \cite{wu2024paramrel} proposed a novel parameter space representation learning framework based on the BFN. In this paper, we first apply the BFN to the field of marked temporal point processes.


\section{Conclusion}
In this paper, we propose a novel approach for MTPP tasks to directly model the marked temporal joint distribution without the limitation of the model assumptions and numerical approximation techniques. In addition, we design the joint noise strategy to capture and reveal the complex interdependence between the timestamp and even types. Extensive experiments on four real-world datasets demonstrate the effectiveness of the proposed approach. In future research, it is of interest to apply the BMTPP to more complex coupled marked temporal data.





\bibliography{IEEE}
\bibliographystyle{ieeetr}

\newpage

 




\vfill

\end{document}